\documentclass[sn-mathphys-num]{sn-jnl}% Math and Physical Sciences Numbered Reference Style
\usepackage{graphicx}%
\usepackage{multirow}%
\usepackage{amsmath,amssymb,amsfonts,bm}%
\usepackage{amsthm}%
\usepackage{mathrsfs}%
\usepackage[title]{appendix}%
\usepackage{xcolor}%
\usepackage{textcomp}%
\usepackage{manyfoot}%
\usepackage{booktabs}%
\usepackage{algorithm}%
\usepackage{algorithmicx}%
\usepackage{algpseudocode}%
\usepackage{listings}%
\theoremstyle{thmstyleone}%
\newtheorem{theorem}{Theorem}%  meant for continuous numbers
\theoremstyle{thmstyletwo}%
\newtheorem{remark}{Remark}%

\theoremstyle{thmstylethree}%
\newtheorem{definition}{Definition}%

\begin{document}

\title[Article Title]{Stagnation in Evolutionary Algorithms: Convergence $\neq$ Optimality}

\author*[1]{\fnm{Xiaojun} \sur{Zhou}}\email{michael.x.zhou@csu.edu.cn}

\affil*[1]{\orgdiv{School of Automation}, \orgname{Central South University}, \orgaddress{\city{Changsha}, \postcode{410083}, \country{China}}}


\abstract{In the evolutionary computation community, it is widely believed that stagnation impedes convergence in evolutionary algorithms, and that convergence inherently indicates optimality. However, this perspective is misleading. In this study,  it is the first to highlight that the stagnation of an individual can actually facilitate the convergence of the entire population, and convergence does not necessarily imply optimality, not even local optimality. Convergence alone is insufficient to ensure the effectiveness of evolutionary algorithms. Several counterexamples are provided to illustrate this argument.}




\maketitle

\section{Introduction}
Stagnation refers to the situation where the best solution found so far remains unchanged over time, which is a common phenomenon in evolutionary computation, as most evolutionary algorithms are stochastic \cite{bonyadi2015stability,doerr2023stagnation}. When stagnation occurs, it is often blamed on bad luck, with the assumption that the evolutionary algorithm has become stuck in a local minimum. As a result, significant efforts have been dedicated to designing new strategies to help existing algorithms escape such traps, or to conducting stability analysis of evolutionary algorithms to ensure convergence.
This leads to the proposition that stagnation impedes convergence, and that convergence inherently signifies optimality.

However, after a thorough  analysis of stagnation, convergence and optimality in this study, it is found that this perspective is misleading.
The main contributions of this study can be summarized as follows:
\begin{enumerate}
  \item This study is the first to highlight that the stagnation of an individual can actually facilitate the convergence of the entire population.
  \item This study is the first to illustrate that convergence does not necessarily imply optimality.
  \item Some counterexamples are provided to demonstrate that existing evolutionary algorithms may converge to non-optimal points.
\end{enumerate}

The remainder of this paper is organized as follows. Section II analyzes convergence and stagnation phenomena in a nominal evolutionary optimizer. Section III examines the optimality properties of evolutionary algorithms. Section IV then provides several counterexamples to demonstrate that convergence does not inherently guarantee optimality. Finally, Section V concludes the paper with key findings and implications.

\section{Analysis of convergence and stagnation}
In this study, the following unconstrained optimization problem is studied:
\begin{eqnarray}
\min_{\bm x \in \Omega} f(\bm x)
\end{eqnarray}
where $\bm x \in \Omega \subseteq \mathcal{R}^n$,  $\Omega$ is a closed and compact set, and $f(\cdot)$ is bounded below.
It is assumed that the optimal solution $\bm x^{*}$ does not lie on the boundary of $\Omega$.

In this section, we consider a population of $N$ individuals and construct a nominal evolutionary optimizer governed by the following dynamics
\begin{equation}
\label{eq:neodynamics}
\bm x_{i} (k+1) = \bm x_{i}(k) + \alpha (\bm x_{j}(k) - \bm x_{i} (k))
\end{equation}
where $\bm{x}_i \in \mathbb{R}^n$ denotes the $i$th individual, $\alpha$ is a parameter, and the $j$th individual ($\bm{x}_j$, with $j \neq i$) can be considered a neighbor of the $i$th individual.

\begin{theorem}
Considering a population with any two individuals $i$ and $j$ with the dynamics described by Eq. (2), if the parameter  $\alpha$  satisfies $0 < \alpha < 1$, the population will converge, \textit{i.e.} $\bm x_{i} (k) = \bm x_{j} (k), k \rightarrow \infty$.
\end{theorem}

\begin{proof}
We begin by defining the state error between individuals \( i \) and \( j \) at time step \( k \) as:
\[
e_{ij}(k) = \bm{x}_i(k) - \bm{x}_j(k),
\]
which quantifies the deviation between their states.

Then, we consider the error at the next time step:
\[
e_{ij}(k+1) = \bm{x}_i(k+1) - \bm{x}_j(k+1).
\]
Using the given update equation Eq. (\ref{eq:neodynamics})  for both individuals:
\begin{align*}
\bm{x}_i(k+1) &= \bm{x}_i(k) + \alpha (\bm{x}_j(k) - \bm{x}_i(k)), \\
\bm{x}_j(k+1) &= \bm{x}_j(k) + \alpha (\bm{x}_i(k) - \bm{x}_j(k)),
\end{align*}
we obtain:
\[
e_{ij}(k+1) = \left[ \bm{x}_i(k) + \alpha (\bm{x}_j(k) - \bm{x}_i(k)) \right] - \left[ \bm{x}_j(k) + \alpha (\bm{x}_i(k) - \bm{x}_j(k)) \right].
\]
Simplifying this expression yields the error dynamics:
\[
e_{ij}(k+1) = (1 - 2\alpha) e_{ij}(k).
\]

This represents a linear dynamical system for the error, with guaranteed convergence when \( |1 - 2\alpha| < 1 \). Solving this inequality gives the stability condition:
\[
0 < \alpha < 1.
\]
Under this condition, the error dynamics are asymptotically stable, ensuring convergence in the population.
\end{proof}

\begin{remark}
This convergence behavior emerges intrinsically from the population dynamics,
without requiring any external control inputs or intervention. Although it only involves two individuals, it can be extended to larger populations.
\end{remark}

On the other hand, the phenomenon of stagnation is a frequently discussed topic in the evolutionary computation community, and it is widely believed that stagnation impedes convergence.
To clarify this point, we assume that the state of individual $j$ has stagnated, \textit{i.e.}, $\bm x_j(k) \equiv \bm x_j$, and then we have
\begin{equation}
\begin{aligned}
\bm x_{i} (k+1) - \bm x_{j}&= (1- \alpha) ( \bm x_{i} (k) - \bm x_j) \\
\end{aligned}
\end{equation}
This system achieves asymptotic convergence when \( |1 - \alpha| < 1 \). Solving this inequality yields a new stability condition:
\[
0 < \alpha < 2.
\]

\begin{remark}
As demonstrated in previous analysis, it is not difficult to find that if an individual has stagnated, the system exhibits relaxed convergence conditions since the admissible range of $\alpha$ is broader. This observation reveals an important characteristic of the dynamics: stagnation of one individual facilitates convergence of the entire population.
\end{remark}

Next, we present another type of convergence under external control, \textit{i.e.}, the acceptance criterion in optimization algorithms.
\begin{theorem}
Let $\bm {Best}(k)$ denote the best solution found so far for the nominal evolutionary optimizer, if the following acceptance criterion is adopted
\begin{equation}
f(\bm {Best}(k+1)) \leq f(\bm {Best}(k))
\end{equation}
then the sequence $\{f(\bm{Best}(k))\}_{k=0}^{\infty}$ will converge.
\end{theorem}

\begin{proof}
Since the sequence ${f(\bm{Best}(k))}_{k=0}^{\infty}$ is monotonically decreasing and the function $f(\cdot)$ is bounded below, by the Monotone Convergence Theorem, the sequence ${f(\bm{Best}(k))}_{k=0}^{\infty}$ converges.
\end{proof}

\begin{remark}
There are other types of convergence results in evolutionary computation \cite{derrac2014analyzing,he2015average,chen2021average}, and in fact, ensuring convergence is generally not challenging; however, convergence does not necessarily imply optimality, not even local optimality,
which is neglected by the majority of existing studies.
\end{remark}

In the next section, the optimality in evolutionary algorithms will be discussed.

\section{Analysis of optimality}
In the evolutionary computation community, there exists a prevalent belief that traditional evolutionary algorithms are inherently susceptible to becoming trapped in local minima.  Consequently, significant research efforts have been devoted to designing new algorithms that can escape such traps. However, the notion of local minima in this context can be misleading. A common but potentially flawed assumption is that the stagnation of the best-found solution across multiple generations necessarily signifies the presence of a local minimum. In reality, this assumption is not necessarily true.

In mathematics, a local minimum of a function \( f(\bm x) \) is defined as follows:
\begin{definition}
Let \( f: \mathcal{R}^n \to \mathcal{R} \) be a real-valued function. A point \( \bm x^{*} \in \mathcal{R}^n \) is a \textbf{local minimum} of \( f(\bm x) \) if there exists a neighborhood \( \mathcal{N}_{\epsilon}(\bm x^{*}) \) of \(\bm x^{*} \) such that for all \( \bm x \in \mathcal{N}_{\epsilon}(\bm x^{*}) \):
\[
f(\bm x^{*}) \leq f(\bm x)
\]
where $\mathcal{N}_{\epsilon}(\bm x^{*}) = \{\bm x \in \mathcal{R}^n: \|\bm x - \bm x^{*}\| \leq \epsilon\}$
\end{definition}

In the evolutionary computation community, a local minimum of a function \( f(\bm x) \) is technically characterized as follows:
\begin{definition}
Let \( f: \mathcal{R}^n \to \mathcal{R} \) be a real-valued function. A point \( \bm x^{*} \in \mathcal{R}^n \) is a \textbf{local minimum} of \( f(\bm x) \) if there does not exist a solution \(\bm x\) such that:
\[
f(\bm x) < f(\bm x^{*}), \;\;\;\; \bm x \in \mathcal{P}_{Alg}(k:k+T)
\]
where \( \mathcal{P}_{Alg}(k:k+T) \)  denotes the population sequence from $k$  to $k+T$  of the evolutionary algorithm, and $T$ is the maximum number of stagnation generations.
\end{definition}

\begin{remark}
If not carefully examined, one might mistakenly conclude that the two definitions are almost the same.
However, there is a significant distinction between them. In mathematics, the concept of a neighborhood is independent of any specific algorithm. In contrast, in the evolutionary computation community, the concept of a neighborhood is dependent on the evolutionary algorithm itself, specifically the operators used in the algorithm.
\end{remark}

In the evolutionary computation community, researchers are keen on constructing benchmark functions, regardless of their homogeneity.
However, this will lead to ``\textbf{overfitting}" in existing evolutionary algorithms since they know their own algorithms' neighborhood quite well, in other words, they are both players and referees.
\section{Experimental results and analysis}
\subsection{Benchmark Functions}
(1) Zhou1 function
\begin{align*}
f_1(\bm{x}) = &\ (x_1 - 1)^2 + \sin^2\left(10^4(x_1 - 1)^2\right) \\
&+ \sum_{i=1}^{n-1} \left[10^4 \left(x_{i+1} - 2x_i^2\right)^2
+ 10^4 \sin^2\left(10^4 \left(x_{i+1} - 2x_i^2\right)\right) \right]
\end{align*}
where the global optimum $x_1^{*} = 1, x_{i+1}^{*} = 2 (x^{*}_{i})^2$ $(1 \leq i \leq n-1)$,  and $f(\bm x^{*}) = 0$.

(2) Zhou2 function
\begin{align*}
f_2(\bm{x}) =\ & (x_1 + 1)^2 + \sin^2\left(10^4 (x_1 + 1)^2 \right)  \\
& + \sum_{i=1}^{n-1} \left[ 10^4 \left( x_{i+1}^2 + 2x_i \right)^2 + 10^4 \sin^2\left( 10^4 \left( x_{i+1}^2 + 2x_i \right)^2 \right) \right]
\end{align*}
where the global optimum $x_1^{*} = -1, x_{i+1}^{*} = -\sqrt{2 x^{*}_{i}}$ $(1 \leq i \leq n-2)$, $x_{i+1}^{*} = \pm \sqrt{2 x^{*}_{i}}$ $( i = n-1)$ and $f(\bm x^{*}) = 0$.

(3) Zhou3 function
\begin{align*}
\ &f_3(\bm{x}) = (x_1 + 1)^2 \left( 1 + \sin^2\left(10^4(x_1 + 1)^2 \right) \right) \\
& + \sum_{i=1}^{n-1} 10^4 \left( x_{i+1}^2 + 2^i x_i \right)^2
\left( 1 + 10^4 \sin^2\left( 10^4 \left( x_{i+1}^2 + 2^i x_i \right)^2 \right) \right) \\
\end{align*}
where the global optimum $x_1^{*} = -1, x_{i+1}^{*} = -\sqrt{2^i x^{*}_{i}}$ $(1 \leq i \leq n-2)$, $x_{i+1}^{*} = \pm \sqrt{2^i x^{*}_{i}}$ $( i = n-1)$ and $f(\bm x^{*}) = 0$.
\subsection{Experiment settings}
To illustrate that convergence does not necessarily guarantee optimality, this study employs the benchmark functions described above and conducts experimental validation using a set of representative evolutionary algorithms.
\begin{itemize}
  \item GL25 \cite{garcia2008global}: global and local real-coded genetic algorithm.
  \item CLPSO \cite{liang2006comprehensive}: comprehensive learning particle swarm optimizer.
  \item LSHADE \cite{tanabe2014improving}: SHADE using linear population size reduction
  \item GWO \cite{mirjalili2014grey}: grey wolf optimizer.
  \item WOA \cite{mirjalili2016whale}: whale optimization algorithm.
  \item HHO \cite{heidari2019harris}: harris hawks optimization.
\end{itemize}
In the experimental setup, all algorithms are implemented using their default parameter settings and executed with 30 independent runs. For simplicity, experiment tests are limited to $n = 3$, and the search range is $[-100,100]$, with the maximum number of stagnation generations at $T = 100, 200, 300, 500, 1000$, respectively. The optimality is evaluated by computing the average gradient norm.

\subsection{Experimental results and analysis}
The experimental results are presented in Table \ref{tab:gradnorm} and Fig. \ref{fig:f1} to Fig. \ref{fig:f3}. It is observed that, for most of the tested algorithms, the average gradient norm remains large and does not exhibit a clear downward trend as the maximum number of stall generations increases, with the exception of LSHADE. Nevertheless, even for LSHADE, the average gradient norm remains significantly high, indicating a considerable deviation from the zero-gradient condition. These findings indicate that the employed evolutionary algorithms do not guarantee convergence to an optimal solution.

\begin{table}[!htbp]
%\scriptsize
%\setlength{\tabcolsep}{3pt} % ĬÈÏ6pt£¬¼õСÁмä¾à
\centering
\caption{Experimental results of the average gradient norm}
\label{tab:gradnorm}
\begin{tabular}{l|c|c|c|c|c|c|c}
\hline
\toprule[1pt]
Fun & $T$ & GL25 & CLPSO & LSHADE & GWO & WOA & HHO \\
\hline
        & 100 & 2.12e+08 & 2.64e+08 & 3.15e+07 & 4.71e+07 & 6.91e+07 & 2.17e+07  \\
        & 200 & 3.07e+08 & 2.37e+08 & 1.36e+07 & 1.87e+07 & 5.87e+07 & 2.46e+07  \\
$f_{1}$& 300 & 2.36e+08 & 2.09e+08 & 4.31e+06 & 2.76e+07 & 4.28e+07 & 2.53e+07  \\
        & 500 & 4.12e+08 & 1.60e+08 & 2.22e+06 & 1.50e+07 & 1.73e+07 & 1.15e+07  \\
        & 1000 & 3.14e+08 & 1.20e+08 & 2.63e+06 & 2.62e+07 & 2.45e+07 & 8.19e+06  \\
\hline
        & 100 & 3.99e+08 & 3.10e+08 & 8.18e+05 & 3.15e+07 & 1.94e+07 & 7.76e+06  \\
        & 200 & 4.98e+08 & 2.44e+08 & 3.08e+03 & 3.70e+07 & 1.59e+07 & 2.46e+06  \\
$f_{2}$& 300 & 3.53e+08 & 1.51e+08 & 3.52e+03 & 2.85e+06 & 1.17e+07 & 3.27e+06  \\
        & 500 & 5.05e+08 & 1.49e+08 & 1.76e+03 & 1.55e+07 & 7.48e+06 & 3.28e+06  \\
        & 1000 & 5.18e+08 & 1.07e+08 & 2.53e+02 & 1.66e+07 & 4.22e+06 & 4.15e+05  \\
\hline
        & 100 & 2.08e+16 & 2.69e+13 & 1.54e+11 & 4.51e+12 & 5.08e+13 & 1.75e+10  \\
        & 200 & 9.05e+15 & 7.83e+12 & 4.79e+06 & 8.52e+10 & 8.56e+14 & 3.12e+10  \\
$f_{3}$& 300 & 7.28e+16 & 1.09e+13 & 2.35e+06 & 4.27e+10 & 9.72e+12 & 7.65e+09  \\
        & 500 & 5.86e+15 & 4.89e+12 & 1.53e+04 & 4.41e+10 & 8.54e+12 & 8.77e+02  \\
        & 1000 & 2.74e+15 & 1.50e+12 & 1.97e+02 & 2.11e+11 & 1.34e+09 & 4.35e+06  \\
\bottomrule[1pt]
\hline
\end{tabular}
\end{table}

\begin{figure}[!htbp]
\centering
\includegraphics[width=0.8\textwidth]{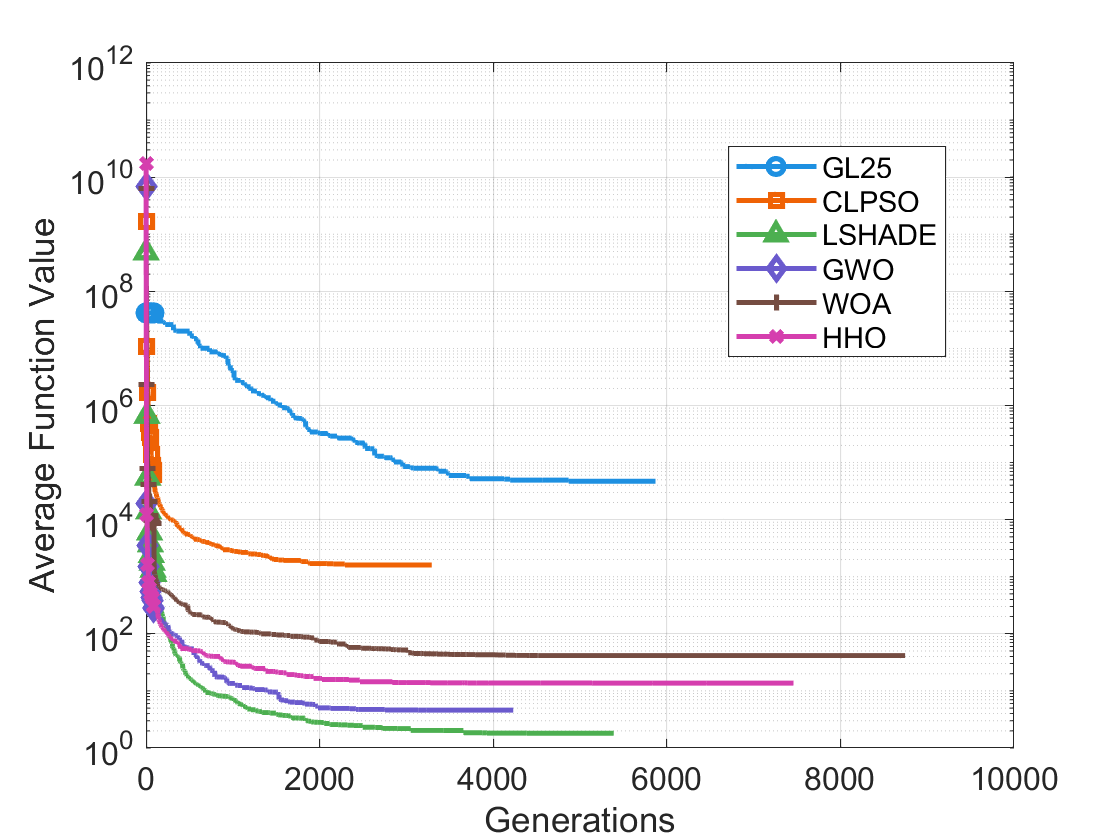}
\caption{Convergence curves for $f_1$ with $T = 1000$}
\label{fig:f1}
\end{figure}

\begin{figure}[!htbp]
\centering
\includegraphics[width=0.8\textwidth]{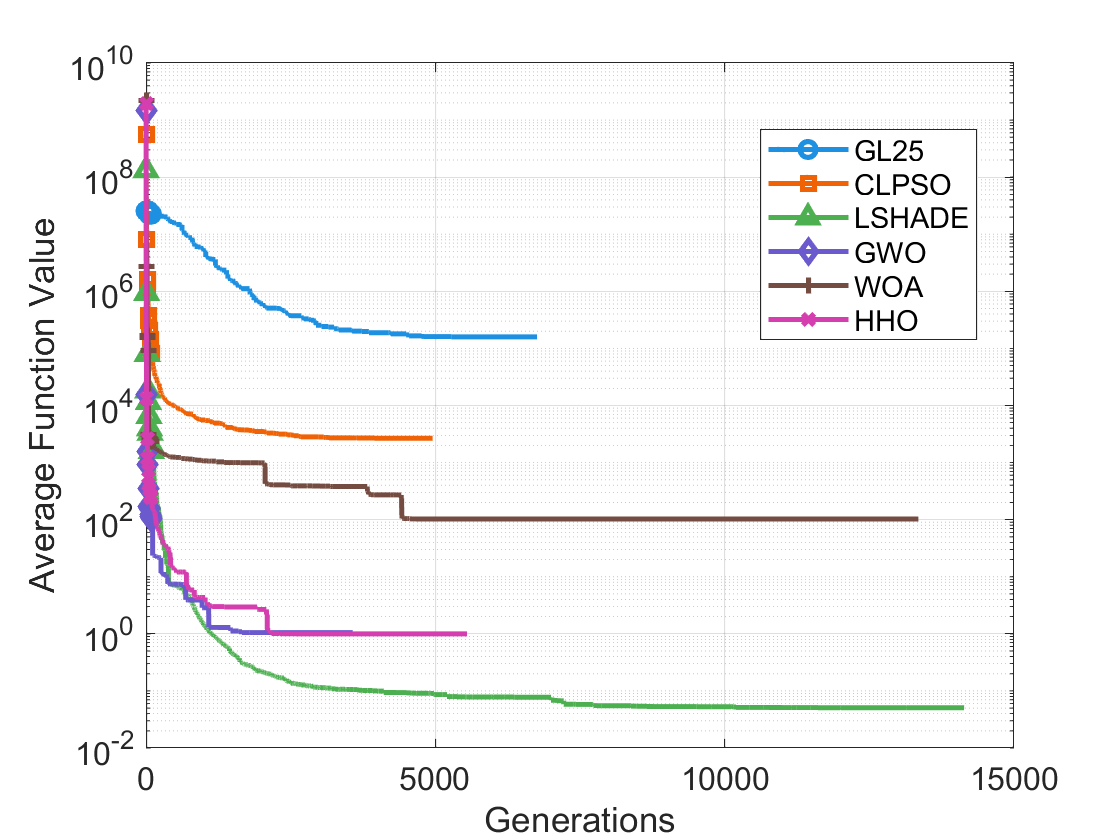}
\caption{Convergence curves for $f_2$ with $T = 1000$}
\label{fig:f2}
\end{figure}

\begin{figure}[!htbp]
\centering
\includegraphics[width=0.8\textwidth]{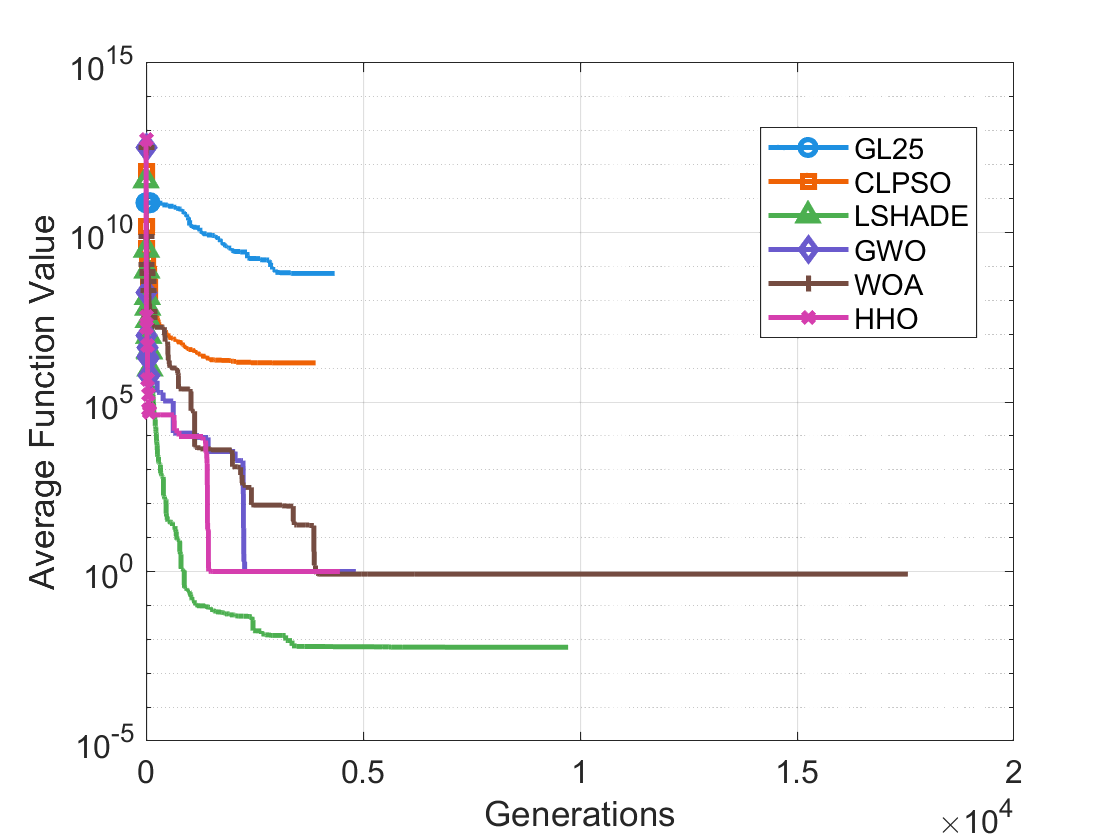}
\caption{Convergence curves for $f_3$ with $T = 1000$}
\label{fig:f3}
\end{figure}

\section{Conclusion}
This paper systematically clarifies three fundamental concepts in evolutionary computation-stagnation, convergence, and optimality-by rigorously delineating their distinctions and supporting our analysis with counterexamples. Our findings demonstrate that, for evolutionary algorithms, convergence does not inherently imply optimality, not even local optimality.

Although existing evolutionary algorithms demonstrate excellent performance on standard benchmark functions, they tend to suffer from overfitting issues. While differential evolution (DE)-like algorithms exhibit competitive results on CEC benchmark functions, this does not necessarily reflect strong global search capability.  An excessive focus on accelerating convergence without ensuring optimality can ultimately hinder optimization performance.

\section*{Acknowledgments}
This work was supported by the National Natural Science Foundation of China under Grant 62273357.


\bibliography{sn-bibliography}% common bib file

\end{document}